%% file: main.tex
\documentclass{article}
\usepackage[preprint]{neurips_2025}

\usepackage[utf8]{inputenc} % allow utf-8 input
\usepackage[T1]{fontenc}    % use 8-bit T1 fonts
\usepackage{hyperref}       % hyperlinks
\usepackage{url}            % simple URL typesetting
\usepackage{booktabs}       % professional-quality tables
\usepackage{amsfonts}       % blackboard math symbols
\usepackage{nicefrac}       % compact symbols for 1/2, etc.
\usepackage{microtype}      % microtypography
\usepackage{xcolor}         % colors

\usepackage{graphicx} 
\usepackage{amsmath}
\usepackage{algorithm}
\usepackage{algorithmic}

\usepackage{caption}
\usepackage{amsthm}
\usepackage{multirow}
\usepackage{multicol} 

\usepackage{caption}       % <<< ESSENTIAL for captionof
\usepackage{wrapfig}
\usepackage{lipsum} % For dummy text
\usepackage{amssymb}
\usepackage{times}         % NeurIPS style might load this or similar fonts
\usepackage{enumitem}

\newtheorem{theorem}{Theorem}

\title{\textsf{ZSMerge}: Zero-Shot KV Cache Compression for Memory-Efficient Long-Context LLMs}

\author{%
Xin Liu$^1$, Xudong Wang$^2$, Pei Liu$^1$, Guoming Tang$^{1}$\thanks{Corresponding author: \texttt{guomingtang@hkust-gz.edu.cn}} \\
$^1$Information Hub, The Hong Kong University of Science and Technology (Guangzhou), China \\
$^2$School of Data Science, The Chinese University of Hong Kong, Shenzhen, China
}

\begin{document}

\maketitle

\input{sections/abstract}
\input{sections/introduction}

\input{sections/related_work}

\input{sections/methdology}

\input{sections/experiment}

\section{Conclusion}  
This study tackles the memory and computational inefficiency of long-context language models with \textsf{ZSMerge}, a dynamic KV cache compression framework. Through head-level memory allocation, residual merging with compensated attention scoring, and architecture-agnostic adaptation, \textsf{ZSMerge} achieves sublinear memory growth while preserving generation quality. Empirical results show an 82\% VRAM reduction at 54K tokens, a threefold increase in inference efficiency under a 5\% compression ratio, and 12–34\% higher text quality metrics over eviction-based baselines, all without task-specific training. Additionally, by enabling efficient deployment on resource-constrained devices, \textsf{ZSMerge} reduces energy consumption and carbon emissions, advancing the sustainability of LLM applications in real-world scenarios.

%%%%%%%%%%%%%%%%%%%%%%%%%%%%%%%%%%%%%%%%%%%%%%%%%%%%%%%%%

% \bibliographystyle{rusnat}
\bibliographystyle{unsrt}
\bibliography{example_paper}

%%%%%%%%%%%%%%%%%%%%%%%%%%%%%%%%%%%%%%%%%%%%%%%%%%%%%%%%%%%%
\newpage
\appendix

\input{sections/appendix}
%%%%%%%%%%%%%%%%%%%%%%%%%%%%%%%%%%%%%%%%%%%%%%%%%%%%%%%%%%%%

% \newpage
% \section*{NeurIPS Paper Checklist}

% \input{sections/checklist}

\end{document}

%% file: sections/abstract.tex
\newcommand{\codeURL}{https://github.com/SusCom-Lab/ZSMerge}

\begin{abstract}
    
The linear growth of key-value (KV) cache memory and quadratic computational complexity in attention mechanisms pose significant bottlenecks for large language models (LLMs) in long-context processing. While existing KV cache optimization methods address these challenges through token pruning or feature merging, they often incur irreversible information loss or require costly retraining. To this end, we propose \textsf{ZSMerge}, a dynamic KV cache compression framework for efficient cache management, featuring three key operations: (1) fine-grained memory allocation guided by multi-dimensional token importance metrics at head-level granularity, (2) a residual merging mechanism that preserves critical context through compensated attention scoring, and (3) a zero-shot adaptation mechanism compatible with diverse LLM architectures without retraining. \textsf{ZSMerge} significantly enhances memory efficiency and inference speed with negligible performance degradation. When applied to LLaMA2-7B, it achieves a 20:1 compression ratio for KV cache retention (reducing memory footprint to 5\% of baseline) while sustaining comparable generation quality and tripling throughput at extreme 54k-token contexts, eliminating out-of-memory failures. These gains extend across diverse LLM families and long-context benchmarks, with \textsf{ZSMerge} consistently surpassing state-of-the-art cache optimization methods across tasks such as summarization, QA, reasoning, and code completion. The code is available at \url{\codeURL}.

\end{abstract}

%% file: sections/introduction.tex
\section{Introduction}
The advancement of the Transformer architecture has revolutionized sequence data processing, with Large Language Models (LLMs) being a prime illustration \citep{vaswani2017attention}. LLMs have achieved remarkable progress, surpassing human performance in diverse applications. While multi-turn dialogue and long-context understanding are critical scenarios~\citep{yi2024survey,gao2023self,li2023loogle,liu2024lost}, handling extended sequences remains challenging due to the linear expansion of the Key-Value (KV) cache, which stores intermediate attention keys and values during generation to avoid redundant computations~\citep{kwon2023efficient}. For instance, deploying a LLaMA2-7B~\citep{sarah2024llama} model will consume about 26GB of GPU memory; when the token length reaches 32K, the KV cache alone occupies 32GB of memory, becoming the primary memory consumption. This, along with the quadratic computational complexity, significantly restricts LLM integration and performance, thus motivating recent research efforts toward optimizing KV cache utilization for enhanced inference efficiency \citep{zhang2024unifying,sun2024shadowkv,wang2024squeezeattention}.

To mitigate the large memory footprint of the KV cache, various optimization strategies have emerged. Some modify the core model architecture~\citep{sarah2024llama, ashkboos2024slicegpt} or employ quantization techniques for lower precision representation~\citep{dettmers2022gpt3, xiao2023smoothquant}. Other methods directly target the KV cache bottleneck during inference using context-aware techniques. Common among these is \textit{sparse approximation} or \textit{token eviction}, retaining only crucial tokens while discarding others~\citep{zhang2024h2o, zhang2024pyramidkv}; however, this can disrupt KV embeddings and cause information loss. To mitigate this, alternative context-aware techniques \textit{merge} or \textit{fuse} information from similar tokens~\citep{dong_get_2024}, often using low-rank compression, though these may introduce architectural overhead or require model retraining.

Ideally, an effective and practical KV cache management strategy should 1) strictly control the memory consumption to break the context length limits, 2) minimize the generation quality degradation, and 3) readily applicable across diverse LLM backbones without necessitating fine-tuning or architectural changes (zero-shot compatibility).

Motivated by these goals, we propose \textsf{ZSMerge}, an efficient, comprehensive, and dynamic zero-shot KV management algorithm, with the following key features:

\begin{itemize}[leftmargin=1em]
    \item \textbf{Fine-grained memory allocation}: \textsf{ZSMerge} makes full use of the historical contribution of tokens, spatio-temporal characteristics, and intrinsic data distribution characteristics to conduct fine-grained, head-level memory management. This adaptability allows it to maintain superior generation quality compared to the sparse strategies and achieve performance comparable to retraining-dependent approaches, even when operating with significantly reduced memory budgets. This ensures effectiveness across diverse long-context tasks.
    
    \item \textbf{Zero-shot and low-cost Integration}: Designed for ease of use, \textsf{ZSMerge} does not introduce any model parameters and thus requires no retraining or model fine-tuning to adapt to different compression ratios or various task types. This also makes it compatible to various mainstream LLM architectures such as LLaMA, Falcon~\citep{almazrouei2023falcon}, Mistral~\citep{jiang2023mistral}, Qwen~\citep{yang2024qwen2}, and Yi~\citep{aiYiOpenFoundation2025}, thereby minimizing deployment overhead.
    
    \item \textbf{Efficient throughput-boosting strategy}: Characterized by linear operational complexity, \textsf{ZSMerge}'s compression mechanism is computationally lightweight. This efficiency brings significant throughput improvement, e.g., it helps achieve over triple the inference throughput compared to the original model at a token length of 54K. Furthermore, we have proved that the KV cache compressed by \textsf{ZSMerge} can effectively preserves the informational contribution of retained tokens, preventing signal degradation even as context length increases substantially.
\end{itemize}

%% file: sections/related_work.tex
\section{Related Work}

KV cache growth in long-context LLMs creates severe memory and computational bottlenecks, motivating optimization strategies across two categories: context-agnostic and contextually adaptive approaches.

\subsection{Context-Agnostic Optimization}

Context-agnostic methods reduce resource consumption through structural or numerical modifications independent of input sequences. Key approaches include: (1) \textbf{Structural compression} via knowledge distillation~\citep{gu2024minillm} or matrix factorization~\citep{ashkboos2024slicegpt}; (2) \textbf{Attention optimization} through head pruning~\citep{voita2019analyzing} and key-value sharing in Multi-Query~\citep{shazeer2019fast} and Grouped-Query Attention~\citep{ainslie2023gqa}; (3) \textbf{Numerical optimization} via post-training quantization~\citep{frantar2022gptq}, hardware-aligned quantization~\citep{chen2024int}, and low-rank approximation~\citep{wang2020linformer}.

\subsection{Contextually Adaptive Optimization}
Context-aware methods directly target KV cache bottlenecks by exploiting attention sparsity~\citep{ji2021distribution, deng2024attention}. Figure~\ref{fig:att_grid} illustrates three main approaches:

\begin{figure*}[t]
    \centering
    \includegraphics[width=1\linewidth]{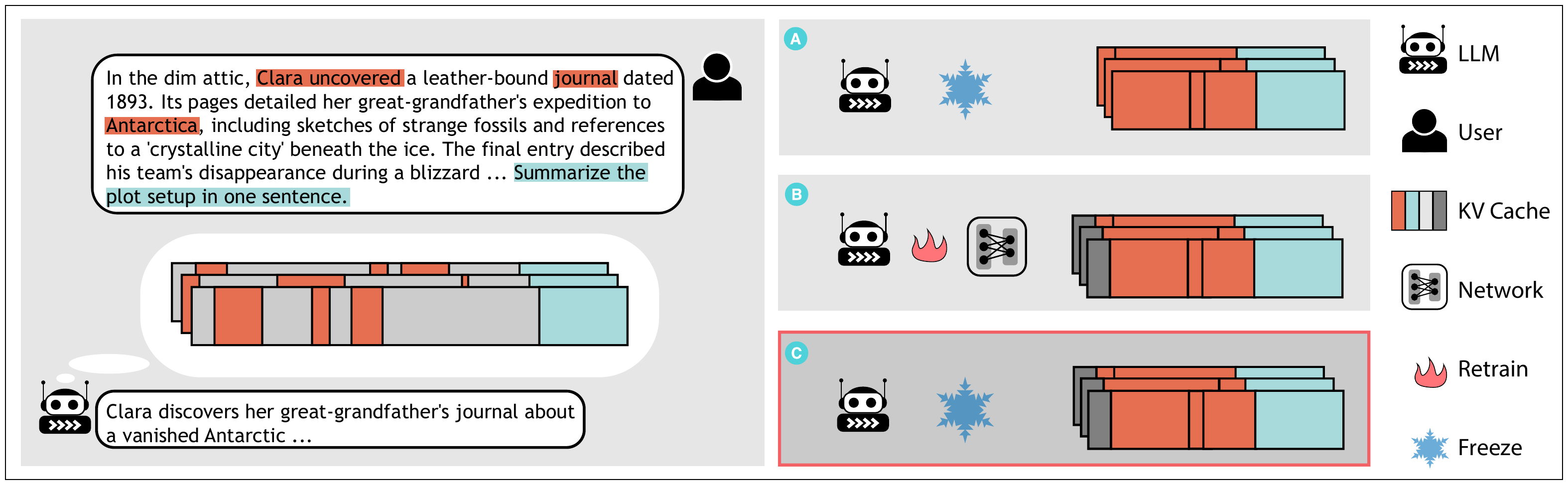}
    \caption{
    Comparative analysis of contextually adaptive KV cache optimization strategies.
    (A) \textbf{Sparse Approximation (Eviction)}: Low-contribution tokens are permanently discarded, reducing memory but risking error propagation.
    (B) \textbf{Token Merging}: Similar tokens are fused to mitigate information loss, but may require fine-tuning or auxiliary networks.
    (C) \textbf{ZSMerge (Proposed)}: Achieves zero-shot "sparse+residual" compression without added parameters or tuning, balancing efficiency and performance.
    }
    \label{fig:att_grid}
\end{figure*}

\textbf{Sparse Approximation (Eviction):} This class of methods {directly tackles the growth of memory and computation by evicting low-contribution tokens}, under the assumption that not all historical context is equally important. For example, StreamingLLM~\cite{xiao2023efficient} identifies "attention sinks" and retains both the initial tokens and a sliding context window. LM-infinite~\cite{han2024lm} uses specialized masks to balance local and global context. SnapKV~\cite{li2024snapkv} performs one-time KV cache pruning during the prefilling stage based on relatively consistent attention patterns, but lacks dynamic adjustment during decoding. H2O~\cite{zhang2024h2o} employs attention score thresholds for dynamic eviction, while PyramidKV~\cite{zhang2024pyramidkv} uses layer-wise attention patterns for progressive caching. Recent advances include FlexPrefill~\cite{laiFlexPrefillContextAwareSparse2025}, which applies context-aware sparse attention specifically during prefilling. Dynamic routing~\cite{roy2021efficient} also fits here. However, eviction strategies inevitably {induce attention distribution drift and error propagation} by permanently discarding tokens~\cite{dao2022flashattention, xiao2023efficient}.

\textbf{Token Merging:} To mitigate information loss from eviction, token merging fuses similar tokens instead of discarding them. Pioneered in vision (e.g., ToMe~\cite{bolya2022token}), it applies to text as well. Dynamic Memory Compression~\cite{nawrot2024dynamic} adaptively merges spatially proximate tokens, often requiring fine-tuning. LESS~\cite{dong_get_2024} uses auxiliary networks to compress tokens with minimal attention impact. ThinK~\cite{xuThinKThinnerKey2024} takes a different approach by applying query-driven pruning at the channel level, orthogonal to token-level compression. While merging {mitigates sparse caching errors by learning residuals (implicitly or explicitly)}, it often {introduces overhead from auxiliary networks or fine-tuning}, complicating deployment.

\textbf{Training-Dependent and System-Level Approaches:} Recent work has explored complementary directions to token-level compression. Training-dependent methods include NSA (Native Sparse Attention)~\cite{yuanNativeSparseAttention2025}, which introduces MLP-based compressor modules trained jointly with the attention mechanism, and MoBA (Mixture of Block Attention)~\cite{luMoBAMixtureBlock2025}, which performs block-level pruning via gating with auxiliary routing networks. System-level approaches like OmniKV~\cite{haoOmniKVDynamicContext2024} focus on dynamic context selection and offloading strategies for efficient long-context processing. InfiniPot~\cite{kimInfiniPotInfiniteContext2024} adopts periodic block-wise compression with a fixed KV cache budget to reduce runtime overhead. While powerful, these methods require retraining, architectural modifications, or specialized system configurations.

\textbf{ZSMerge (Proposed):} Addressing the limitations of prior approaches, we propose ZSMerge. It {achieves zero-shot compression competitive using the "sparse+residual" method without added parameters}, {while demonstrating superior generalization across task domains and compression ratios}. ZSMerge offers a parameter-free, tuning-free approach that is complementary to training-dependent and system-level methods, balancing efficiency and generation performance.

%% file: sections/methdology.tex
\section{\textsf{ZSMerge} Methodology}
\label{sec:methodology}

Building upon the conceptual foundation of \textsf{ZSMerge} outlined earlier, this section delineates the technical framework underpinning its zero-shot compression mechanism. We propose a four-component methodology consisting of adaptive budget allocation, context-sensitive contribution evaluation, residual token merging, and stabilized attention projection. The interplay of these components ensures robust generalization across compression ratios and task domains without auxiliary parameters or fine-tuning, addressing limitations of prior eviction and merging strategies.

\subsection{Preliminaries}  
Consider an $L$-layer transformer with multi-head attention mechanisms. For a target attention head at decoding step $T$, the cached Key and Value matrices are defined as:
\begin{equation}
\begin{aligned}
    \mathbf{K}_T = [\mathbf{k}_1, \mathbf{k}_2, \ldots, \mathbf{k}_T]^\top \in \mathbb{R}^{T \times d},~
    \mathbf{V}_T = [\mathbf{v}_1, \mathbf{v}_2, \ldots, \mathbf{v}_T]^\top \in \mathbb{R}^{T \times d},
\end{aligned}
\end{equation}
where $\mathbf{k}_t, \mathbf{v}_t \in \mathbb{R}^d$ represent the Key/Value vectors for the $t$-th token. For the query vector $\mathbf{q}_T \in \mathbb{R}^d$ at position $T$, the scaled dot-product attention computes output $\mathbf{o}^{(T)}$ via:  
\begin{equation}\label{att_score}
\begin{aligned}
    a_t^{(T)} = \frac{\exp(\mathbf{q}_T^\top \mathbf{k}_t / \sqrt{d})}{\sum_{i=1}^T \exp(\mathbf{q}_T^\top \mathbf{k}_i / \sqrt{d})}, ~
    \mathbf{o}^{(T)} = \sum_{t=1}^T a_t^{(T)} \mathbf{v}_t.
\end{aligned}
\end{equation}
This formulation establishes the baseline for analyzing cache compression effects on attention distribution fidelity.  

\subsection{Budget Allocation}

\textsf{ZSMerge} strategically compresses the original $T$-length cache into a budget $B \ll T$ through tripartite allocation:  
\begin{equation}
    B = B_p + B_c + B_r,
\end{equation}
where $B_p$, $B_c$, and $B_r$ govern proximity maintenance, context preservation, and residual absorption, respectively.
This tripartite division enables a nuanced approach to compression.
Our empirical findings (see Appendix~\ref{sec:hyperparams}) indicate that allocating the largest share of the budget to $B_c$ and $B_p$ is generally effective, as it prioritizes the retention of core information and local context.
In contrast, $B_r$ is typically assigned a smaller budget to handle fine-grained residual adjustments.
The compressed cache matrices therefore integrate three complementary components:
\begin{equation}
    \mathbf{K}_B = [\mathbf{K}_{p} \| \mathbf{K}_{c} \| \mathbf{K}_{r}], \quad
    \mathbf{V}_B = [\mathbf{V}_{p} \| \mathbf{V}_{c} \| \mathbf{V}_{r}],
\end{equation}
with $\|$ denoting row-wise concatenation. 

The \textit{proximity component} $(\mathbf{K}_{p}, \mathbf{V}_{p})$ preserves the latest $B_p$ tokens, capturing local attention patterns. The \textit{context component} $(\mathbf{K}_{c}, \mathbf{V}_{c})$ retains top-$B_c$ tokens ranked by contribution scores $\mathbf{s}^{(T)} \in \mathbb{R}^T$, which quantify contextual saliency.
The \textit{residual component} $(\mathbf{K}_{r}, \mathbf{V}_{r})$ dynamically merges $B_r$ historically evicted tokens through attention-weighted aggregation transformations. The \textit{residual component} $(\mathbf{K}_r, \mathbf{V}_r)$ maintains $B_r$ dedicated token slots, representing a compressed summary of previously evicted tokens.  During generation, the residual cache is updated dynamically, progressively merging newly expelled tokens into the existing compressed representation. This configuration subsumes eviction-based methods when $B_r=0$.

\subsection{Contribution Evaluation}

The contribution scores $\mathbf{s}^{(T)} \in \mathbb{R}^T$ dynamically quantify each token's cumulative influence across decoding steps. For the $t$-th token, its score $s_t^{(T)}$ evolves through exponential decay integration of attention activations:
\begin{equation}\label{eq:score}
s_t^{(T)} = 
\begin{cases}
\lambda s_t^{(T-1)} + a_{t}^{(T)}, & T > 0 \\
0,                             & \text{otherwise}
\end{cases},
\end{equation}

where the decay factor $\lambda \in [0,1]$ acts as a temporal discounting parameter analogous to reinforcement learning credit assignment, controlling the exponential decay of historical attention contributions.

In practice, we find that the performance is not highly sensitive to the precise value of $\lambda$. We therefore fix $\lambda=0.98$ throughout this paper for simplicity, which yields a smooth balance between long-term and short-term contributions.

\subsection{Residual Token Merging}

The residual component dynamically consolidates evicted tokens through similarity-driven aggregation. When merging a candidate token $(\mathbf{k}_t, \mathbf{v}_t)$ into $\mathbf{K}_r$, we adopt the following three-step procedure:
\begin{enumerate}[leftmargin=3em]
    \item Select the most compatible residual slot via maximum dot production:
    \begin{equation}\label{eq:hatr}
    \hat{r} = \mathop{\arg\max}\limits_{r \in \{1,\dots,B_r\}} \mathbf{k}_r^\top \mathbf{k}_t% \frac{\mathbf{k}_r^\top \mathbf{k}_t}{\|\mathbf{k}_r\| \|\mathbf{k}_t\|}
    \end{equation}
    
    \item Update the selected slot via incremental mean aggregation: \begin{equation}\label{eq:kvr}
    \mathbf{k}_{\hat{r}} \gets \frac{w_{\hat{r}} \mathbf{k}_{\hat{r}} + \mathbf{k}_t}{w_{\hat{r}} + 1}, \quad
    \mathbf{v}_{\hat{r}} \gets \frac{w_{\hat{r}} \mathbf{v}_{\hat{r}} + \mathbf{v}_t}{w_{\hat{r}} + 1}
    \end{equation}
    
    \item Increment fusion count: $w_{\hat{r}} \gets w_{\hat{r}} + 1$
\end{enumerate}

Figure~\ref{fig:kv_operation} illustrates the dynamic cache evolution under budget parameters $B_r=2$, $B_c=4$, and $B_p=3$ during sequence length expansion from $T=12$ to $T=17$.

\subsection{Attention Output Stabilization}  

\begin{figure}[t] % Use a figure environment to hold the minipages together and allow placement control
  \begin{minipage}[t]{0.56\textwidth} % Minipage for the algorithm (adjust width)
    \vspace{0pt} % Helps with top alignment
    % Option A: Keep the algorithm float (might work, use [H] if needed)
    \begin{algorithm}[H]
    \caption{ZSMerge Online Compression}\label{alg:kvu}
    \begin{algorithmic}[1]
    \STATE \textbf{Input}: Budgets $B_p$, $B_c$, $B_r$, decay $\lambda$
    \STATE \textbf{Init}: $\mathbf{K}_B \gets (\emptyset, \emptyset, \emptyset)$, $\mathbf{V}_B \gets (\emptyset, \emptyset, \emptyset)$ 
    
    \FOR{decoding step $T = 1,2,\ldots$}
        \STATE Compute attention scores $a_{T,t}$ via \eqref{att_score}
        \STATE Update contribution scores $\mathbf{s}^{(T)}$ via \eqref{eq:score}
        \STATE $\mathbf{K}_p \gets \mathbf{K}_p \cup \{\mathbf{k}_T\}$, $\mathbf{V}_p \gets \mathbf{V}_p \cup \{\mathbf{v}_T\}$
        
        \IF{$|\mathbf{K}_p| > B_p$} 
            \STATE Evict oldest token $(\mathbf{k}_\text{old}, \mathbf{v}_\text{old})$ from $\mathbf{K}_p$, $\mathbf{V}_p$
            \STATE $\mathbf{K}_c \gets \mathbf{K}_c \cup \{\mathbf{k}_\text{old}\}$, $\mathbf{V}_c \gets \mathbf{V}_c \cup \{\mathbf{v}_\text{old}\}$
            
            \IF{$|\mathbf{K}_c| > B_c$}
                \STATE Evict lowest-score token $\hat{c} = \arg\min_{c} s_c^{(T)}$  
                
                \IF{$|\mathbf{K}_r| + 1 \leq B_r$} 
                
                    \STATE $\mathbf{K}_r \gets \mathbf{K}_r \cup \{\mathbf{k}_{\hat{c}}\}$, $\mathbf{V}_r \gets \mathbf{V}_r \cup \{\mathbf{v}_{\hat{c}}\}$
                \ELSE
                    \STATE Merge token $(\mathbf{k}_{\hat{c}}, \mathbf{v}_{\hat{c}})$ into $\mathbf{K}_r$ via \eqref{eq:hatr}-\eqref{eq:kvr} 
                \ENDIF
            \ENDIF
        \ENDIF
    \ENDFOR
    \end{algorithmic}
    \end{algorithm}
  \end{minipage}% <-- The % is important to avoid extra space between minipages
  \hfill % Pushes the minipages apart
  \begin{minipage}[t]{0.42\textwidth} % Minipage for the other content (adjust width)
    \vspace{6pt} % Helps with top alignment
    
        \centering % Optional: centers the content within the wrapfigure width
        \includegraphics[width=\linewidth]{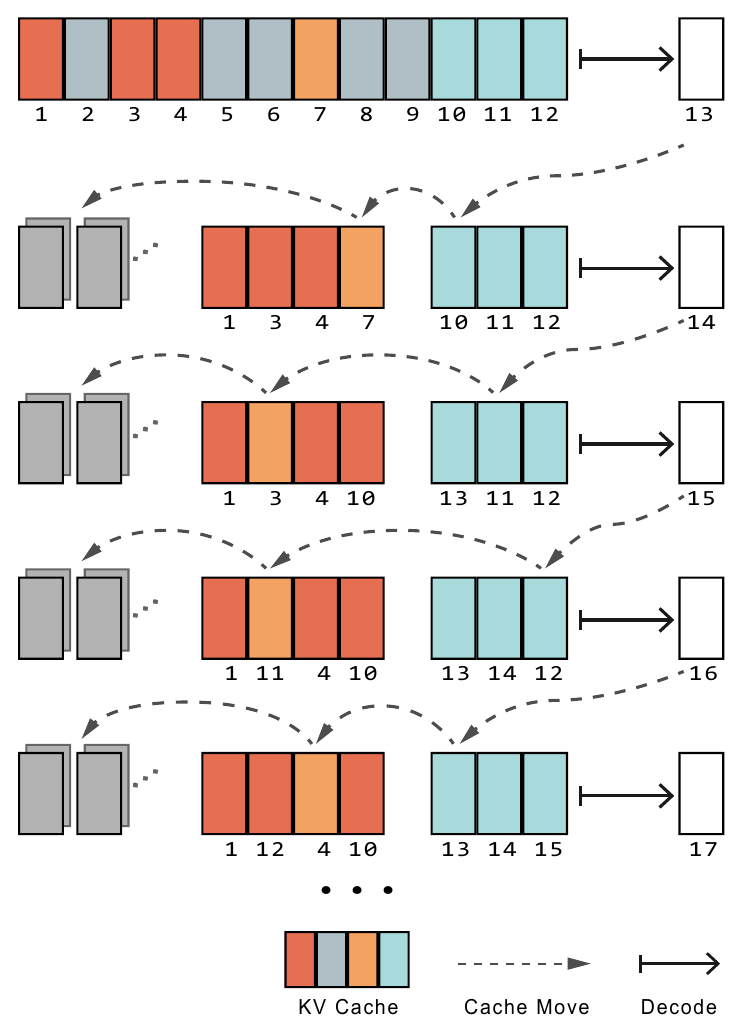} % Use \linewidth here, it refers to the width set for wrapfigure
        \captionof{figure}{KV cache operation process} %  with $B_r=2$, $B_c=4$, $B_p=3$, and sequence length expansion $T=12 \rightarrow 17$: residual merging (2 slots), context preservation (top-4) and proximity maintenance (latest 3).
    \label{fig:kv_operation}
    
  \end{minipage}
\end{figure}

Algorithm \ref{alg:kvu} progressively merges low-saliency tokens through dynamic cache updates. To more accurately reconstruct the attention distribution from compressed representations, we introduce a compensated attention scoring mechanism. The revised attention computation evolves from Eq.~\ref{att_score} to:
\begin{equation}\label{att_score2}
\begin{aligned}
    \hat{a}_{t}^{(T)} = \frac{\exp\left(\mathbf{q}_T^\top\mathbf{k}_t/\sqrt{d} + \alpha\log w_t\right)}{\sum_{i=1}^T \exp\left(\mathbf{q}_T^\top\mathbf{k}_i/\sqrt{d} + \alpha\log w_i\right)},
    ~
    \mathbf{\hat{o}}^{(T)} = \sum_{t=1}^T \hat{a}_{t}^{(T)}\mathbf{v}_t,
\end{aligned}
\end{equation}

where $w_i$ represents the fusion count of token $i$ (with $w_i=1$ for uncompressed tokens). The hyperparameter $\alpha \in [0,1]$ provides a soft transition, with $\alpha = 0$ degenerating to an eviction-like policy. Its concrete influence is demonstrated in the appendix, and we fix $\alpha=1$ throughout our experiments.

With Eq. \ref{att_score2}, we address two key challenges:
\begin{itemize}[leftmargin=1em]
    \item \textbf{Representation Bias Correction}: Merged tokens aggregate multiple historical vectors via Eq.~\ref{eq:kvr}, creating a mismatch between their key vectors ($\mathbf{k}_r$) and the original value distribution. The $\log w_j$ term compensates for this representation shift.
    \item \textbf{Attention Mass Conservation}: The compensation term preserves the relative attention mass between compressed and uncompressed tokens, ensuring residual compensation does not suppress critical uncompressed components.
\end{itemize}

\begin{theorem}\label{thm:main}
For any query step $T$ and uncompressed token $i \notin B_r$, the revised attention score satisfies $\hat{a}_i^{(T)} \geq {a}_i^{(T)}$, where ${a}_i^{(T)}$ is the original attention score from Eq.~\ref{att_score}.
\end{theorem}

The proof is provided in Appendix~\ref{sec:proof}.

This theoretical guarantee ensures that uncompressed tokens retain their relative dominance in attention allocation despite cache compression, even as the denominator accounts for upper-bounded contributions from compressed tokens. The compensation mechanism effectively preserves attention mass for critical context tokens while preventing over-amplification of merged token scores, thereby maintaining the integrity of the original attention distribution under compression.

%% file: sections/experiment.tex
\section{Experimental Evaluations}

\label{experiment}

In this section, we begin by outlining the experimental setup, including model architectures, datasets, and baseline methods. We then assess efficiency improvements in terms of memory and speed. Next, we provide a detailed numerical error analysis to highlight the role of representation bias correction. Finally, we evaluate generation quality under diverse cache budgets, across multiple task types, and over different base model series, while comparing against a broad set of baselines.

\subsection{Experiment Setup}

\textbf{Model Coverage}: Our evaluation encompasses diverse modern LLM architectures for broad compatibility. Core experiments use LLaMA2-7B, Falcon-7B, and Mistral-7B-Instruct on NVIDIA A800-80GB GPU. We extend evaluation to LLaMA-3.1-8B-Instruct, Qwen2.5-7B-Instruct, and Yi-6B (Appendix~\ref{sec:extended_arch}), spanning different attention mechanisms including Multi-Query Attention (MQA).

\textbf{Benchmark Diversity}: We use synthetic datasets for efficiency evaluation and multiple public benchmarks for quality assessment: LongBench~\citep{bai2024longbench} (21 tasks, 6 categories), XSum~\citep{narayanDontGiveMe2018} (summarization), InfiniteBench~\citep{zhang-etal-2024-bench} (100K+ tokens), and GSM-Infinite-8k~\citep{zhouGSMInfiniteHowYour2025} (mathematical reasoning). Extended results are in Appendix~\ref{appendix:exp-details}.

\textbf{Baselines}: We compare \textsf{ZSMerge} against representative KV cache management methods:
\begin{itemize}[leftmargin=1em]
\item \textbf{FullKV}: Stores all key-value pairs at every layer, providing the vanilla baseline with maximal memory and compute overhead, but serving as the upper bound for task performance.
\item \textbf{StreamingLLM (Stream)}~\citep{xiao2023efficient}: Retains both early prompt tokens and a fixed sliding window of recent tokens. By designating "attention sinks," it ensures semantic anchors remain available. This static rule is efficient and robust but lacks adaptivity to task-specific token importance.
\item \textbf{SnapKV}~\citep{li2024snapkv}: Performs a one-time cache pruning during prefilling, leveraging early attention distributions to discard less relevant tokens. It eliminates runtime overhead and achieves efficiency, but cannot adapt if token importance shifts during generation.
\item \textbf{H2O}~\citep{zhang2024h2o}: Maintains adaptivity via cumulative attention thresholds that dynamically retain heavy hitters alongside recency-biased tokens. Layer-wise averaging of attention scores captures persistent importance, enabling balanced compression while preserving key semantics.
\item \textbf{LESS}~\citep{dong_get_2024}: Introduces dynamic KV state synthesis through recurrent merging. It combines recency preservation with similarity-based compression, then applies attention rectification to correct distortions introduced by merging. This allows sublinear cache growth without severe loss of contextual coherence.
% \item \textbf{ZSMerge (Ours)}: Advances merging-based methods through heuristic, zero-shot compression without parametric training. %It enforces strict cache bounds via two-phase selection: (1) preserve sinks and recency tokens, (2) cluster intermediate tokens using distance-aware similarity. Revised attention scoring accounts for merged token densities, reducing distortion and improving robustness across compression ratios.
\end{itemize}

\textbf{Extended Baseline Coverage}: Beyond the core baselines, our comprehensive evaluation includes additional state-of-the-art methods to provide thorough comparative analysis. These include OmniKV~\citep{haoOmniKVDynamicContext2024} for attention-guided compression, InfLLM~\citep{xiaoInfLLMTrainingFreeLongContext2024} for infinite-length modeling, Minference~\citep{jiangMInference10Accelerating2024} for efficient attention approximation, and FlexPrefill~\citep{laiFlexPrefillContextAwareSparse2025} for flexible prefilling strategies. Detailed comparisons with these advanced baselines on challenging benchmarks such as InfiniteBench and GSM-Infinite-8k are presented in Appendix~\ref{sec:extended_arch}, demonstrating ZSMerge's competitive performance against the latest compression techniques.

\subsection{Inference Efficiency Gain}

To demonstrate the benefits of \textsf{ZSMerge} in improving inference efficiency, we conducted two proof-of-concept experiments. The first (Figure~\ref{fig:memory_usage}) compares the performance of full KV caching and \textsf{ZSMerge} under increasing sequence lengths in a specific inference case. The second (Table~\ref{table:latency}) evaluates \textsf{ZSMerge} against full KV caching and other baseline methods under various workloads, including different sequence lengths, batch sizes, and model sizes. Below, we summarize the memory and throughput improvements achieved by \textsf{ZSMerge}.

\begin{figure}[t] % Use a figure environment to hold the minipages together and allow placement control
  \begin{minipage}[t]{0.72\textwidth} % Minipage for the algorithm (adjust width)
    \vspace{0pt} % Helps with top alignment
        \centering % Optional: centers the content within the wrapfigure width
        \includegraphics[width=\linewidth]{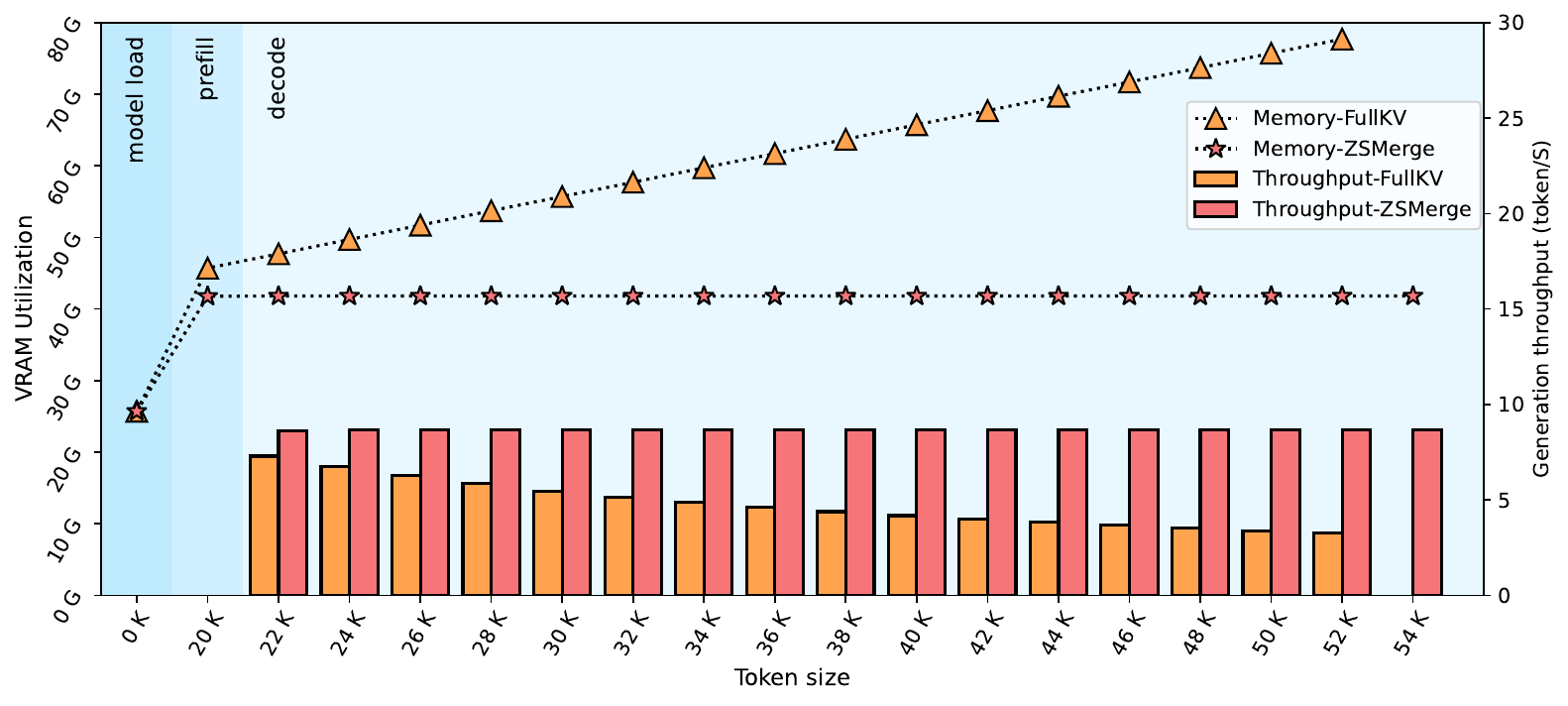} % Use \linewidth here, it refers to the width set for wrapfigure
  \end{minipage}% <-- The % is important to avoid extra space between minipages
  \hfill % Pushes the minipages apart
  \begin{minipage}[t]{0.28\textwidth} % Minipage for the other content (adjust width)
    \vspace{0pt} % Helps with top alignment
        \captionof{figure}{VRAM Utilization and Decoding Throughput Across Sequence Lengths. ZSMerge enforces constant memory footprint (43GB) and sustains 9 tokens/sec decoding rate beyond 54K tokens, eliminating out-of-memory (OOM) via dynamic KV cache compression.}
    \label{fig:memory_usage}

  \end{minipage}
\end{figure}

\begin{table}[t]
\caption{Workload-Scalable KV Cache Compression: ZSMerge Outperforms Baselines in Throughput (tokens/sec) and Latency (seconds) Across Sequence Lengths and Batch Sizes.}
\label{table:latency}
% \vskip 0.15in
\begin{center}
\begin{sc}
\scalebox{0.8}{\begin{tabular}{lcccccc}
\toprule
    \multirow{2}{*}{Seq.length} &
    \multirow{2}{*}{Batch size} &
    \multirow{2}{*}{Model size} &
    \multicolumn{4}{c}{throughput(tokens/S) $\Uparrow$ / latency(S) $\Downarrow$}\\
    \cline{4-7}
    \multicolumn{3}{c}{} & FullKV & H2O (5\%) & Less (5\%)  & ZSMerge (5\%)  \\

\midrule
    1024+1024 & 8 & 7B & \textbf{177.8} / \textbf{46.1} & 104.9 / 78.1 & 48.8 / 167.9 & 161.6 / 50.7 \\
    2048+2048 & 8 & 7B & 110.8 / 147.9 & 72.2 / 227.0 & 25.1 / 654.2 & \textbf{163.2} / \textbf{100.4} \\
    2048+2048 & 16 & 7B & 133.1 / 246.2 & 86.1 / 380.1 & OOM & \textbf{281.9} / \textbf{178.4}\\
    % 4096+4096 & 4 & 7B & 62.4 / 262.1 & 43.9 / 372.5 & 15.0 / 1086.2 & \textbf{77.0} / \textbf{212.7}\\
    4096+4096 & 8 & 13B & OOM & OOM & OOM & \textbf{110.8} / \textbf{295.7} \\
\midrule
    2048+2048 & 2 & 13B & \textbf{43.0} / \textbf{95.2} & 25.5 / 160.4 & 12.7 / 322.7 & 31.3 / 131.0 \\
    2048+2048 & 4 & 13B & 59.7 / 137.2 & 40.7 / 201.2 & 16.5 / 497.5 & \textbf{61.5} / \textbf{133.3} \\
    4096+4096 & 4 & 13B & 37.1 / 441.8 & 24.9 / 657.9 & OOM & \textbf{60.0} / \textbf{273.2} \\
    4096+4096 & 16 & 13B & OOM & OOM & OOM & \textbf{178.2} / \textbf{367.6}\\
\bottomrule
\end{tabular}}
\end{sc}
\end{center}
% \vskip -0.1in
\end{table}

\subsubsection{Memory Reduction}

\paragraph{Specific Case Analysis}
In the baseline setup, the LLaMA2-7B model required 25GB of VRAM for parameter loading, with an additional 20GB KV cache generated during the prefill phase for 20K tokens. This linear growth in KV cache size, at 1MB per token, led to OOM errors as sequence lengths approached 54K tokens.

\textsf{ZSMerge}, constrained by an 18K token cache budget, reduced KV cache size by 10\% during the prefill phase and maintained VRAM usage at a constant 43GB during decoding. This prevented the baseline's linear memory growth (up to 79GB) and completely eliminated OOM errors, enabling efficient long-context processing.

\subsubsection{Throughput Improvement}

Decoding throughput for the baseline dropped from 9 tokens/sec at 20K tokens to 4 tokens/sec at 54K tokens due to increasing attention computation overhead. \textsf{ZSMerge}, in contrast, maintained a consistent throughput of 9 tokens/sec across the same sequence length range by dynamically merging less relevant tokens while preserving critical attention information.
% This constant throughput and stable memory usage arise from ZSMerge's dynamic adjustment at each decoding step, in contrast to once-for-all pruning applied only during prefilling, as in the SnapKV method.

\textsf{ZSMerge} consistently outperformed baselines across diverse workloads, achieving superior throughput while avoiding OOM errors that plagued other methods. Notably, for memory-intensive scenarios (e.g., 13B model, 4096+4096 seq, batch 16), \textsf{ZSMerge} was the only method capable of processing requests (178.2 tokens/s), demonstrating its scalability advantage through dynamic compression.

\subsection{Numerical Error Analysis}

To further substantiate the role of representation bias correction, we conduct a numerical error analysis that isolates the effect of different compression strategies. In particular, we measure the relative error of attention outputs under varying cache budget constraints, comparing pure eviction-based compression against our residual merging approach. This setup directly reflects the practical conditions under which compression-induced representation shifts occur.

\begin{figure}[t]
  \begin{minipage}[t]{0.56\textwidth}
    \vspace{0pt}
        \centering
        \includegraphics[width=\linewidth]{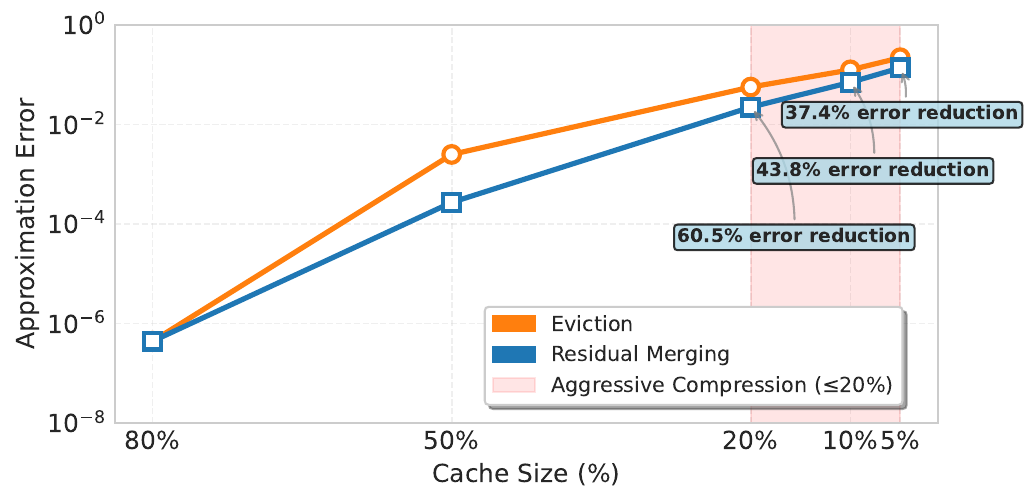}
  \end{minipage}
  \hfill
  \begin{minipage}[t]{0.4\textwidth}
    \vspace{0pt}
        \captionof{figure}{Numerical error analysis comparing eviction vs. residual merging across compression ratios. The log-scale visualization demonstrates that ZSMerge's residual merging consistently reduces approximation error relative to pure eviction, with benefits amplified under aggressive compression ($\le 20\%$ cache size).}
    \label{fig:error_analysis}
  \end{minipage}
\end{figure}

As shown in Figure~\ref{fig:error_analysis}, under aggressive compression scenarios ($\leq 20\%$ cache size), \textsf{ZSMerge}'s residual merging demonstrates substantial error reduction: 60.5\% at 20\% cache size, 43.8\% at 10\%, and 37.4\% at 5\%. The most dramatic improvement occurs at 50\% compression, where residual merging achieves 89.1\% error reduction ($2.72\times 10^{-4}$ vs.\ $2.50\times 10^{-3}$). This empirical validation confirms that the residual compensation mechanism effectively preserves attention distribution coherence, preventing the error propagation characteristic of pure eviction strategies.

The results show that residual merging consistently yields lower approximation error than eviction-based methods, especially under tight memory budgets where aggressive compression is required. This analysis empirically validates the representation-bias-correction hypothesis underlying \textsc{ZSMerge}. Additional validation on the efficiency of residual compensation and its connection to Eq.~\ref{att_score2} is provided in Appendix~\ref{sec:hyperparams}.

\subsection{Inference Quality Loss}

\subsubsection{Impact of Cache Budget}

{%
  \setlength\intextsep{0pt}% Remove float separation
\begin{wraptable}[16]{r}{0.6\textwidth}
  \centering
\caption{Text Generation Quality under KV Cache Compression: ROUGE-1/2/L Scores for LLaMA2-7B and Falcon-7B Across 20\%, 10\%, and 5\% Cache Budgets}
\label{table:rouge}
% \vskip 0.15in
\begin{center}
\begin{sc}
\scalebox{0.84}{\begin{tabular}{lcc}
\toprule
method & LLaMA2-7B & Falcon-7B\\

\midrule
Full KV & 30.59 / 11.34 / 25.50 & 27.06 / 8.79 / 22.39 \\
\midrule
20\% H2O & 30.83 / 11.43 / 25.71 & 24.18 / 7.47 / 20.14 \\
20\% LESS & 30.47 / 11.23 / 25.44 & 24.90 / 7.97 / 20.76 \\
20\% ZSMerge & \textbf{31.62} / \textbf{12.40} / \textbf{26.42} & \textbf{25.19} / \textbf{8.69} / \textbf{21.34} \\
\midrule
10\% H2O & 30.18 / 11.32 / 25.28 & 13.03 / 3.16 / 11.02 \\
10\% LESS & 30.74 / 11.22 / 25.58 & 8.99 / 2.55 / 7.50 \\
10\% ZSMerge & \textbf{31.83} / \textbf{12.47} / \textbf{26.75} & \textbf{20.92} / \textbf{6.74} / \textbf{17.50} \\
\midrule
5\% H2O & 28.92 / 10.81 / 24.35 & 12.02 / 2.08 / 10.36 \\
5\% LESS & 29.98 / 11.09 / 25.02 & 7.75 / 1.15 / 6.76 \\
5\% ZSMerge & \textbf{30.60} / \textbf{11.67} / \textbf{25.72} & \textbf{15.04} / \textbf{3.29} / \textbf{12.73} \\

\bottomrule
\end{tabular}}
\end{sc}
\end{center}
\end{wraptable}

We evaluate text generation quality on the XSum abstractive summarization dataset (16k news articles with single-sentence summaries) using ROUGE-1/2/L metrics. Experiments compare three KV-cache compression methods across two 7B-parameter models (\textsc{LLaMA2-7B} and \textsc{Falcon-7B}) under 20\%, 10\%, and 5\% cache budgets. \textsc{H2O} follows its original design, balancing recent tokens and heavy hitters, while \textsc{LESS} uses its recommended merging network trained on the C4 dataset~\cite{raffel2020exploring}. For \textsf{ZSMerge}, budget allocation matches \textsc{LESS} ($B_r=2$, with $B_p$ and $B_c$ sharing the remaining budget), and other hyperparameters are fixed as previously described.

Our evaluation demonstrates that \textsf{ZSMerge} achieves superior text generation quality across varying compression ratios and model architectures compared to existing KV cache compression methods. As shown in Table~\ref{table:rouge}, when compressing the KV cache to 5\% of its original size, \textsf{ZSMerge} maintains near-original performance on LLaMA2-7B, even slightly exceeding the uncompressed baseline in ROUGE scores, while H2O and LESS exhibit measurable degradation under the same conditions. This advantage becomes more pronounced at higher compression ratios, where \textsf{ZSMerge}'s performance preservation significantly outperforms both baselines.

On Falcon-7B, a model with distinct architectural characteristics, \textsf{ZSMerge} demonstrates stronger generalization than parameter-dependent approaches. While LESS suffers severe quality degradation-likely due to its reliance on C4 dataset training. \textsf{ZSMerge} retains over half of the baseline performance even at extreme 5\% compression. In contrast, H2O struggles to balance recent and heavy-hitter tokens effectively, particularly for Falcon's attention patterns.

The results highlight two critical trends: First, higher compression ratios amplify \textsf{ZSMerge}'s relative advantages, as its heuristic merging mechanism better preserves semantically critical tokens compared to purely eviction policies. Second, its training-free nature enables consistent robustness across different model architectures, avoiding the domain adaptation challenges faced by data-driven methods such as LESS that require training data to tune auxiliary modules.
}

\subsubsection{Generalization across Task Types}

{%
  \setlength\intextsep{0pt}% Remove float separation
\begin{wraptable}[26]{r}{0.564\textwidth}
  \centering
\caption{Results on \textit{LongBench} benchmark including code completion (CODE), few-shot learning(FSHOT), multi-document QA (MDQA), single-document QA (SDQA), summarization (SUMM), and synthetic reasoning (SYNC) tasks.}
\label{table:longbench}
% \vskip 0.15in
\begin{center}
\begin{sc}
\scalebox{0.7}{\begin{tabular}{lcccccc}
\toprule
    Method & \rotatebox{30}{CODE} & \rotatebox{30}{FSHOT} & \rotatebox{30}{MDQA} & \rotatebox{30}{SDQA} & \rotatebox{30}{SUMM} & \rotatebox{30}{SYNC} \\
\midrule
\midrule
& \multicolumn{6}{c}{\textit{LlaMA2-7B}} \\
    % \cline{2-7}
\midrule
 FullKV & 65.19 & 52.32 & 11.48 & 17.50 & 15.15 & 5.01 \\
    \cline{2-7}
& \multicolumn{6}{l}{\textit{Cache size=512}} \\
H2O & 19.63 & 7.37 & 4.48 & 4.05 & 3.51 & 2.16 \\
Stream & 61.26 & 47.63 & 8.35 & 11.60 & 7.07 & 4.32 \\
SnapKV & \textbf{63.06} & \textbf{51.51} & 10.33 & \textbf{15.85} & \textbf{12.28} & \textbf{5.50} \\
ZSMerge & 63.02 & 51.42 & \textbf{10.36} & 15.67 & 12.17 & 5.34 \\
    \cline{2-7}
& \multicolumn{6}{l}{\textit{Cache size=1024}} \\
H2O & 29.12 & 20.34 & 5.92 & 7.47 & 6.70 & 2.97 \\
Stream & 62.94 & 49.98 & 8.16 & 11.80 & 7.18 & 4.45 \\
SnapKV & \textbf{64.45} & 52.02 & 11.30 & \textbf{16.65} & \textbf{13.30} & \textbf{5.09} \\
ZSMerge & 64.41 & \textbf{52.06} & \textbf{11.37} & 16.63 & 13.23 & 5.02 \\
\midrule
\midrule
& \multicolumn{6}{c}{\textit{Mistra-7B}} \\
    % \cline{2-7}
\midrule
FullKV & 62.10 & 63.09 & 37.34 & 38.94 & 26.14 & 66.67 \\
    \cline{2-7}
& \multicolumn{6}{l}{\textit{Cache size=512}} \\
Stream & 57.90 & 53.68 & 27.04 & 23.92 & 19.80 & 32.17 \\
SnapKV & 60.46 & \textbf{62.24} & \textbf{34.75} & \textbf{36.86} & \textbf{22.18} & \textbf{64.83} \\
ZSMerge & \textbf{60.48} & 62.17 & 34.62 & 36.83 & 22.09 & 64.50 \\
    \cline{2-7}
& \multicolumn{6}{l}{\textit{Cache size=1024}} \\
Stream & 60.13 & 56.01 & 27.80 & 25.10 & 21.25 & 34.17 \\
SnapKV & \textbf{61.60} & 62.36 & \textbf{35.49} & 37.30 &\textbf{ 23.73} & \textbf{66.50} \\
ZSMerge & 61.59 & \textbf{62.42} & 35.43 & \textbf{37.31} & 23.64 & \textbf{66.50 }\\

\bottomrule
\end{tabular}}
\end{sc}
\end{center}
\end{wraptable}
% \vskip -0.1in

We evaluate the effectiveness of \textsf{ZSMerge} on the LongBench benchmark, which covers a diverse set of task types, including code completion, few-shot learning, multi-document QA, single-document QA, summarization, and synthetic reasoning tasks. Experiments are conducted on two backbone models: LLaMA2-7B and Mistral-7B, with cache size constraints set to 512 and 1024. For the Mistral-7B model, the H2O baseline encountered out-of-memory (OOM) errors and could not be included in the comparison. More details about the experimental setup and validation results on specific datasets are provided in Appendix~\ref{appendix:exp-details}.

Overall, the results in Table~\ref{table:longbench} show that \textsf{ZSMerge} consistently matches or closely tracks the performance of FullKV across most task types, despite operating under strict cache constraints. Compared to Stream and H2O, \textsf{ZSMerge} demonstrates substantially higher accuracy, especially in complex and memory-intensive tasks like multi-document QA and summarization. Interestingly, while SnapKV benefits from a one-time pruning strategy during the prefilling phase which allows it to retain a larger cache during decoding, \textsf{ZSMerge} still achieves comparable or even slightly better performance in several tasks, thanks to its adaptive two-phase compression and positional-aware merging strategy.

These findings highlight \textsf{ZSMerge}'s strong generalization and robustness across heterogeneous workloads, making it a promising solution for real-world deployment in memory-constrained scenarios. Additional results with richer benchmarks and broader baseline comparisons are provided in Appendix~\ref{sec:extended_arch}.
}

%% file: sections/appendix.tex
\section{Proof of Theorem~\ref{thm:main}}
\label{sec:proof}

\begin{proof}[Proof of Theorem~\ref{thm:main}]
Let $r \in B_r$ be a residual slot merging $\{r_1,...,r_{w_r}\}$ tokens. For compressed tokens, we establish an upper bound on their attention numerator:
\begin{equation}
\begin{aligned}
\text{num}(\hat{a}_r^{(T)}) &= \exp\left(\mathbf{q}_T^\top\mathbf{k}_r/\sqrt{d} + \alpha\log w_r\right) \\
&\leq w_r \exp\left(\mathbf{q}_T^\top\mathbf{k}_r/\sqrt{d}\right) \\
&= w_r \exp\left(\frac{1}{w_r}\sum_{m=1}^{w_r}\mathbf{q}_T^\top\mathbf{k}_{r_m}/\sqrt{d}\right) \\
&\leq \sum_{m=1}^{w_r}\exp\left(\mathbf{q}_T^\top\mathbf{k}_{r_m}/\sqrt{d}\right),
\end{aligned}
\end{equation}

where the first inequality uses $\alpha\leq1$, and the second applies Jensen's inequality~\citep{jakvsetic2016exponential} to the convex exponential function. For uncompressed tokens ($w_i=1$), we derive:
\begin{equation}
\begin{aligned}
\hat{a}_i^{(T)} &= \frac{\exp(\mathbf{q}_T^\top\mathbf{k}_i/\sqrt{d})}{\sum_{t\notin B_r}\exp(\mathbf{q}_T^\top\mathbf{k}_t/\sqrt{d}) + \sum_{r\in B_r}\text{num}(\hat{a}_r^{(T)})} \\
&\geq \frac{\exp(\mathbf{q}_T^\top\mathbf{k}_i/\sqrt{d})}{\sum_{t=1}^T\exp(\mathbf{q}_T^\top\mathbf{k}_t/\sqrt{d})} = {a}_i^{(T)}.
\end{aligned}
\end{equation}

as the denominator contains compressed tokens' upper-bounded contributions. \qedhere
\end{proof}

\section{Implementation Details}
\label{sec:implementation-details}

This section presents the implementation details of \textsf{ZSMerge}.

The KV cache compression framework is built upon the Transformers library. To minimize deviations from the original framework and reduce redevelopment complexity, only the forward propagation function was replaced globally. As a result, a single process cannot simultaneously hold two instances with different compression modes. However, the compression mode can be easily switched without creating a new instance by calling the $change\_mode$ method.

Our framework currently supports replacing the $scaled\_dot\_product\_attention$ function for the LLaMA, Falcon, and Mistral model families, as this operation is widely used across various inference scenarios.

The initialization of the attention score $s$ based on the full history of attention scores during the prefilling stage imposes a substantial computational burden. In certain long-sequence tasks (such as the LongBench experiments), we introduce a hyperparameter, $window\_size$, to limit the range of timesteps considered during the initialization of $s$, following the approach used in SnapKV. This optimization has a minimal impact on generation quality but significantly accelerates the prefilling process.

% \subsection{Extended Related Works, Discussions, and Limitations}
% \label{sec:extended-related-works}

% This work aims to explore and validate a lightweight KV cache compression method. The transformations are deliberately constrained within a specific scope to avoid complicating the analysis of the core idea. However, this inevitably limits deeper investigation into broader system-level scenarios, such as server deployment performance in real-world machine clusters. Additionally, the potential benefits of combining this method with other orthogonal techniques, such as FlashAttention and PagedAttention, remain to be explored. Furthermore, integrating the approach into mainstream inference frameworks like vLLM or SGLong presents a valuable but challenging direction for future work.

\section{Extended Experiments}
\label{appendix:exp-details}
\label{sec:extended-experiments}
\subsection{Latency and Throughput Evaluation Across Sequence Lengths and Batch Sizes}

Table~\ref{table:latency_full} presents additional experimental results on workload-scalable KV cache compression. The validation of the LESS and H2O frameworks was conducted using the code provided at \url{https://github.com/hdong920/LESS}. The results were obtained from a single experiment run, as the observed conclusions were clear and consistent. In short-sequence and low-batch-size settings, the FullKV method shows slight advantages, as compression introduces additional computational overhead. However, in other scenarios, \textsf{ZSMerge} demonstrates significant performance gains, even when compared to other compression methods.

\begin{table}[H]
\caption{Workload-Scalable KV Cache Compression: \textsf{ZSMerge} Outperforms Baselines in Throughput (tokens/sec) and Latency (seconds) Across Sequence Lengths and Batch Sizes.}
\label{table:latency_full}
% \vskip 0.15in
\begin{center}
\begin{sc}
\scalebox{0.82}{\begin{tabular}{lcccccc}
\toprule
    \multirow{2}{*}{Seq.length} &
    \multirow{2}{*}{Batch size} &
    \multirow{2}{*}{Model size} &
    \multicolumn{4}{c}{throughput(tokens/S) $\Uparrow$ / latency(S) $\Downarrow$}\\
    \cline{4-7}
    \multicolumn{3}{c}{} & FullKV & H2O (5\%) & Less (5\%)  & ZSMerge (5\%)  \\
    1024+1024 & 4 & 7B & 117.5 / 34.9 & 83.5 / 49.0 & 22.7 / 180.6 & 81.5 / 50.3 \\
    1024+1024 & 8 & 7B & 177.8 / 46.1 & 104.9 / 78.1 & 48.8 / 167.9 & 161.6 / 50.7 \\
    2048+2048 & 8 & 7B & 110.8 / 147.9 & 72.2 / 227.0 & 25.1 / 654.2 & 163.2 / 100.4 \\
    2048+2048 & 16 & 7B & 133.1 / 246.2 & 86.1 / 380.1 & OOM & 281.9 / 178.4 \\
    4096+4096 & 4 & 7B & 62.4 / 262.1 & 43.9 / 372.5 & 15.0 / 1086.2 & 77.0 / 212.7 \\
    4096+4096 & 8 & 7B & 65.5 / 500.3 & OOM & OOM & 146.7 / 223.2 \\
    4096+4096 & 16 & 7B & OOM & OOM & OOM & 271.6 / 241.3 \\
    8192+4096 & 4 & 7B & 38.9 / 420.8 & OOM & OOM & 74.3 / 220.4 \\
    8192+8192 & 4 & 7B & 33.1 / 989.3 & OOM & OOM & 78.3 / 418.5 \\
    8192+4096 & 8 & 7B & OOM & OOM & OOM & 132.0 / 248.2 \\
    8192+8192 & 8 & 7B & OOM & OOM & OOM & 142.6 / 459.6 \\
\midrule
    256+256 & 2 & 13B & 62.5 / 8.2 & 49.2 / 10.4 & 31.1 / 16.5 & 30.7 / 16.7 \\
    512+512 & 2 & 13B & 61.4 / 16.7 & 43.5 / 23.6 & 27.1 / 37.8 & 31.1 / 32.9 \\
    1024+1024 & 2 & 13B & 54.4 / 37.7 & 33.2 / 61.6 & 16.8 / 121.8 & 31.1 / 65.8 \\
    2048+2048 & 2 & 13B & 43.0 / 95.2 & 25.5 / 160.4 & 12.7 / 322.7 & 31.3 / 131.0 \\
    2048+2048 & 4 & 13B & 59.7 / 137.2 & 40.7 / 201.2 & 16.5 / 497.5 & 61.5 / 133.3 \\
    4096+4096 & 4 & 13B & 37.1 / 441.8 & 24.9 / 657.9 & OOM & 60.0 / 273.2 \\
    4096+4096 & 8 & 13B & OOM & OOM & OOM & 110.8 / 295.7 \\
    4096+4096 & 16 & 13B & OOM & OOM & OOM & 178.2 / 367.6 \\
    4096+8192 & 16 & 13B & OOM & OOM & OOM & 666.3 / 196.7 \\
    4096+4096 & 32 & 13B & OOM & OOM & OOM & 397.5 / 329.7 \\
    8192+8192 & 4 & 13B & OOM & OOM & OOM & 617.0 / 53.1 \\
    8192+8192 & 8 & 13B & OOM & OOM & OOM & 642.3 / 102.0 \\
    8192+4096 & 16 & 13B & OOM & OOM & OOM & 466.7 / 140.4 \\
    8192+8192 & 16 & 13B & OOM & OOM & OOM & 812.6 / 161.3 \\
    8192+8192 & 32 & 13B & OOM & OOM & OOM & OOM \\
    % 16384+8192 & 8 & 13B & OOM & OOM & OOM & xxx \\
    % 16384+16384 & 8 & 13B & OOM & OOM & OOM & xxx \\
\bottomrule
\end{tabular}}
\end{sc}
\end{center}
% \vskip -0.1in
\end{table}

\subsection{Hyperparameter Sensitivity Validation}
\label{sec:hyperparams}

As a training-free framework for KV cache compression, our method introduces several hyperparameters to enhance flexibility and provide a smooth transition to classical sparsity-based methods. We conduct comprehensive sensitivity analysis to offer empirical recommendations for practical deployment scenarios.

\textbf{Budget Distribution Strategy}: Our framework follows a hierarchical budget allocation strategy. First, the proximity maintenance budget $B_p$ is controlled by the cache tail ratio ($B_p / B$). Then, a small portion of the remaining budget is allocated to the residual budget $B_r$, controlled by the cache dense parameter ($B_r / (B - B_p)$). Finally, the remaining budget is distributed to the context preservation budget $B_c$.

\textbf{Experimental Setup}: We conduct sensitivity analysis on LLaMA2-7B using the XSUM summarization task. We fix an anchor configuration and systematically vary each hyperparameter to isolate its individual impact on performance, measured by ROUGE-1, ROUGE-2, and ROUGE-L scores.

\begin{itemize}[leftmargin=1em]
    \item \textbf{Proximity Maintenance Ratio ($B_p / B$)}: This parameter governs the allocation between proximity maintenance and context preservation. Our analysis reveals that extreme partitions significantly hinder performance. Values below 0.3 or above 0.7 show notable degradation, with ROUGE-1 scores dropping by 4-8\% at the extremes (0.1 and 0.9). The optimal range lies between 0.3 and 0.7, with peak performance around 0.5, consistent with established methods like H2O. We recommend setting this ratio to 0.5 for balanced performance.

    \item \textbf{Residual Budget Ratio ($B_r / (B - B_p)$)}: The residual budget demonstrates the effectiveness of our token merging operation. When $B_r = 0$, the method degrades to a pure eviction strategy, showing 6-11\% performance drops across all ROUGE metrics. Small positive values (0.01-0.03) provide stable benefits, with 0.02 showing optimal performance. Higher values (0.05-0.08) compress the context budget excessively, leading to performance degradation of 7-18\%. We recommend setting this ratio to 0.02 to achieve stable benefits while preserving sufficient context budget.

    \item \textbf{Scale Factor ($\alpha$)}: This parameter controls the influence of merged cache tokens. Our analysis shows that increasing the scale factor from 0.0 to 1.0 progressively improves performance, with ROUGE scores improving by 1-5\%. Setting $\alpha = 0$ degenerates the method to standard sparse approaches, while $\alpha = 1.0$ provides optimal performance. This validates the effectiveness of our token merging operation. We restrict the value to 1.0 in our standard experiments.
\end{itemize}

\begin{figure}[H]
    \centering
    \includegraphics[width=0.8\linewidth]{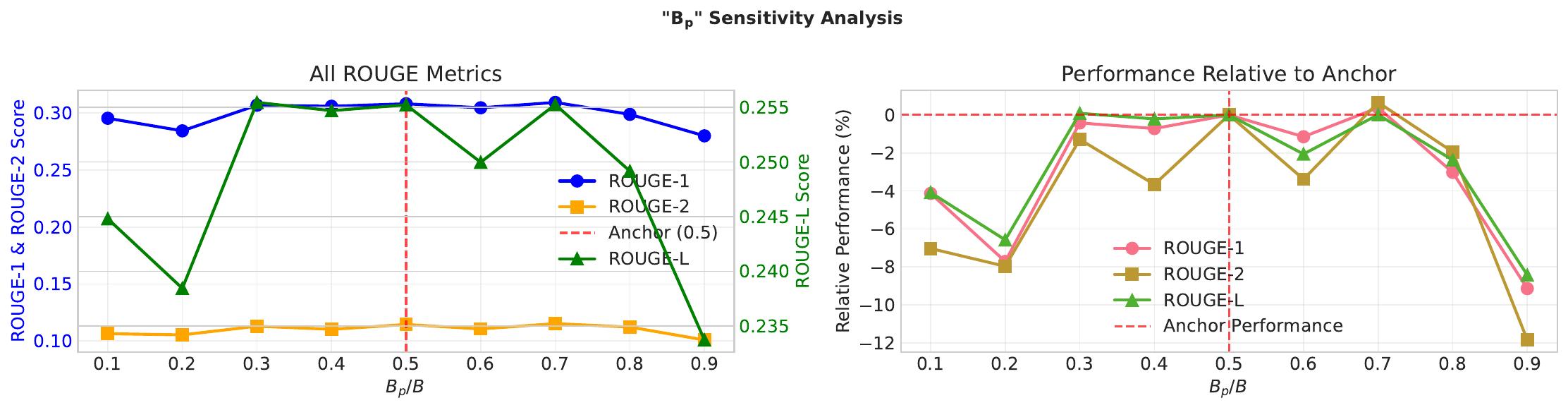}
    \\[0.5em]
    \includegraphics[width=0.8\linewidth]{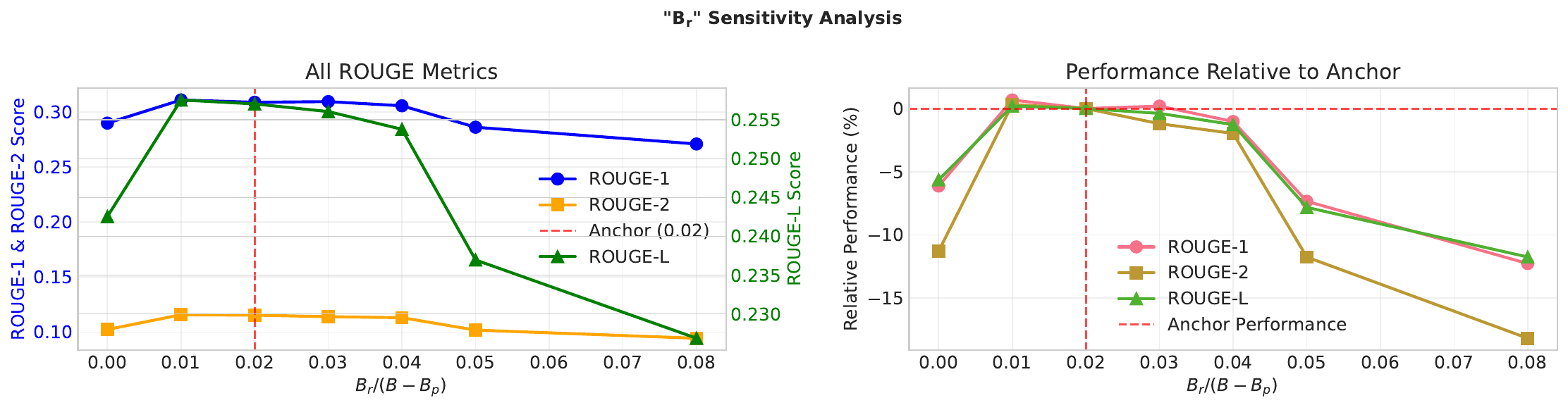}
    \\[0.5em]
    \includegraphics[width=0.8\linewidth]{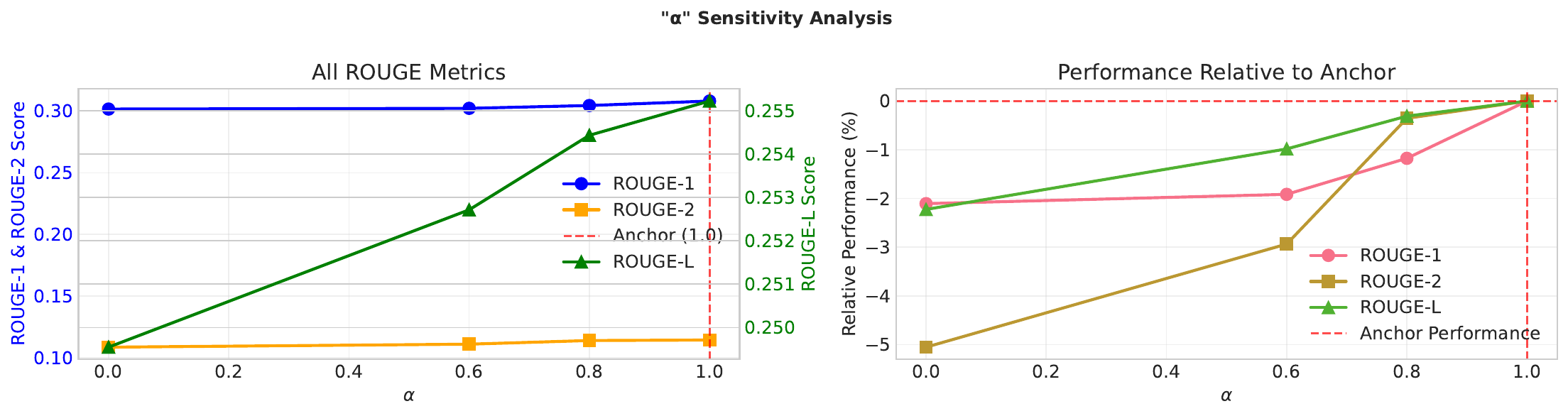}
    \caption{Hyperparameter Sensitivity Analysis: (top) Proximity Maintenance Ratio ($B_p/B$), (middle) Residual Budget Ratio ($B_r/(B-B_p)$), (bottom) Scale Factor ($\alpha$)}
    \label{fig:hyperparameter_sensitivity}
\end{figure}

\textbf{Key Findings}: Figure~\ref{fig:hyperparameter_sensitivity} illustrates the sensitivity patterns across all three hyperparameters. The analysis confirms that \textsf{ZSMerge} demonstrates robust performance across most hyperparameter configurations. The method shows particular sensitivity to extreme budget allocations but maintains stable performance within recommended ranges. The effectiveness of token merging is clearly demonstrated through the consistent improvements observed when enabling residual budgets and scale factors, distinguishing our approach from pure eviction-based methods.

\textbf{Practical Recommendations}: For deployment scenarios, we recommend the following configuration: $B_p/B = 0.5$, $B_r/(B-B_p) = 0.02$, and $\alpha = 1.0$. This configuration provides robust performance while maintaining compatibility with existing sparse attention frameworks.

\subsection{Details on LongBench benchmark}
\textbf{Dataset}: We evaluate \textsf{ZSMerge} on the LongBench benchmark, a comprehensive suite for assessing long-context understanding in LLMs. The benchmark comprises 21 tasks spanning six categories:
\begin{itemize}[leftmargin=*,noitemsep]
\item Single-document QA
\item Multi-document QA
\item Summarization
\item Few-shot learning
\item Synthetic tasks
\item Code completion
\end{itemize}
The datasets (English and Chinese) feature context lengths of 5,000-15,000 tokens and are standardized for automated evaluation~\citep{bai2024longbench}.

\textbf{Evaluation Framework}: Our experiments utilize the benchmarking methods implemented in the KVCache-Factory repository (\url{https://github.com/Zefan-Cai/KVCache-Factory}). This framework supports various KV cache compression methods, including PyramidKV, SnapKV, StreamingLLM, and H2O. It is compatible with attention mechanisms such as Flash Attention v2 and SDPA, allowing for efficient evaluation under different memory constraints .

\textbf{Models and Configuration}: We conduct experiments on two backbone models: LLaMA2-7B and Mistral-7B. To simulate memory-constrained scenarios, we set the cache size constraints to 512 and 1024 tokens. Notably, for the Mistral-7B model, the H2O baseline encountered out-of-memory (OOM) errors and was excluded from the comparison.

\begin{table}[h]
\caption{Performance Comparison of KV Cache Compression Methods on LongBench Tasks}
\centering

\scalebox{0.51}
{\begin{tabular}{lccccccccccccccccccccc}
\hline
\toprule
Method & \rotatebox{90}{2wikimqa} & \rotatebox{90}{dureader} & \rotatebox{90}{gov\_report} & \rotatebox{90}{hotpotqa} & \rotatebox{90}{lcc} & \rotatebox{90}{lsht} & \rotatebox{90}{multi\_news} & \rotatebox{90}{multifieldqa\_en} & \rotatebox{90}{multifieldqa\_zh} & \rotatebox{90}{musique} & \rotatebox{90}{narrativeqa} & \rotatebox{90}{passage\_count} & \rotatebox{90}{passage\_retrieval\_en} & \rotatebox{90}{passage\_retrieval\_zh} & \rotatebox{90}{qasper} & \rotatebox{90}{qmsum} & \rotatebox{90}{repobench-p} & \rotatebox{90}{samsum} & \rotatebox{90}{trec} & \rotatebox{90}{triviaqa} & \rotatebox{90}{vcsum} \\
\hline
\midrule
& \multicolumn{21}{c}{\textit{LlaMA2-7B}} \\
    % \cline{2-7}
\midrule
 FullKV & 10.54 & 23.35 & 27.16 & 7.77 & 68.13 & 20.25 & 3.01 & 23.93 & 18.78 & 4.26 & 17.33 & 1.50 & 5.52 & 8.00 & 9.94 & 20.55 & 62.25 & 32.11 & 68.00 & 88.92 & 9.89 \\
    \cline{2-22}
& \multicolumn{21}{l}{\textit{Cache size=512}} \\
H2O & 7.33 & 3.72 & 1.18 & 4.81 & 22.06 & 0.00 & 1.79 & 5.62 & 1.90 & 2.05 & 5.53 & 2.36 & 4.03 & 0.08 & 3.16 & 10.99 & 17.19 & 3.08 & 18.50 & 7.90 & 0.10 \\
Stream & 9.43 & 13.97 & 0.98 & 6.60 & 64.73 & 16.67 & 0.29 & 17.21 & 11.82 & 3.41 & 11.45 & 1.88 & 5.12 & 5.96 & 5.93 & 18.74 & 57.79 & 30.34 & 56.00 & 87.51 & 8.28 \\
SnapKV & 10.70 & 17.87 & 18.18 & 8.33 & 66.35 & 17.25 & 2.61 & 21.76 & 17.00 & 4.41 & 17.39 & 2.75 & 6.26 & 7.50 & 7.26 & 19.98 & 59.76 & 33.96 & 67.50 & 87.34 & 8.37 \\
ZSMerge & 10.72 & 17.98 & 17.62 & 8.34 & 66.32 & 17.25 & 2.65 & 21.52 & 16.59 & 4.41 & 17.30 & 2.75 & 6.26 & 7.00 & 7.29 & 20.05 & 59.71 & 33.62 & 67.50 & 87.29 & 8.37 \\
    \cline{2-22}
& \multicolumn{21}{l}{\textit{Cache size=1024}} \\
H2O & 8.49 & 6.41 & 4.25 & 6.08 & 38.60 & 0.50 & 2.91 & 11.14 & 4.84 & 2.70 & 6.34 & 2.23 & 6.67 & 0.00 & 7.55 & 19.51 & 19.65 & 30.86 & 19.50 & 30.52 & 0.14 \\
Stream & 9.29 & 13.33 & 0.99 & 6.49 & 66.61 & 17.00 & 0.91 & 16.99 & 12.10 & 3.51 & 11.47 & 1.62 & 3.92 & 7.81 & 6.66 & 19.05 & 59.27 & 31.87 & 62.50 & 88.54 & 7.77 \\
SnapKV & 10.67 & 21.78 & 22.54 & 8.02 & 67.64 & 18.75 & 2.74 & 22.56 & 18.04 & 4.73 & 17.49 & 2.54 & 5.98 & 6.75 & 8.52 & 20.58 & 61.26 & 32.97 & 68.00 & 88.38 & 7.33 \\
ZSMerge & 10.67 & 22.12 & 22.05 & 7.97 & 67.63 & 18.75 & 2.82 & 22.39 & 18.07 & 4.73 & 17.55 & 2.54 & 5.76 & 6.75 & 8.49 & 20.67 & 61.19 & 33.13 & 68.00 & 88.38 & 7.39 \\
\midrule
& \multicolumn{21}{c}{\textit{Mistral-7B-Instruct-v0.3}} \\
FullKV & 39.01 & 32.38 & 34.89 & 49.37 & 61.56 & 40.25 & 27.83 & 52.88 & 32.26 & 28.58 & 29.07 & 5.50 & 98.00 & 96.50 & 41.58 & 25.77 & 62.63 & 47.51 & 76.00 & 88.59 & 16.08 \\
    \cline{2-22}
& \multicolumn{21}{l}{\textit{Cache size=512}} \\
Stream & 31.86 & 17.64 & 22.10 & 41.05 & 59.37 & 18.50 & 23.20 & 29.91 & 15.60 & 17.60 & 24.21 & 6.00 & 81.00 & 9.50 & 25.95 & 20.25 & 56.42 & 43.79 & 65.50 & 86.95 & 13.65 \\
SnapKV & 38.72 & 24.02 & 25.85 & 49.53 & 60.32 & 37.75 & 24.96 & 54.05 & 28.27 & 26.72 & 28.79 & 5.00 & 96.00 & 93.50 & 36.34 & 24.08 & 60.60 & 46.75 & 75.00 & 89.44 & 13.82 \\
ZSMerge & 38.56 & 23.68 & 25.47 & 49.67 & 60.21 & 37.75 & 24.85 & 53.98 & 28.38 & 26.58 & 29.01 & 5.00 & 95.00 & 93.50 & 35.94 & 24.22 & 60.76 & 46.67 & 75.00 & 89.28 & 13.81 \\
    \cline{2-22}
& \multicolumn{21}{l}{\textit{Cache size=1024}} \\
Stream & 32.65 & 17.17 & 24.59 & 43.35 & 61.04 & 21.25 & 25.48 & 31.16 & 16.46 & 18.03 & 24.81 & 5.50 & 82.50 & 14.50 & 27.95 & 20.81 & 59.21 & 45.59 & 68.50 & 88.71 & 14.11 \\
SnapKV & 38.86 & 26.13 & 28.30 & 49.14 & 61.34 & 38.25 & 26.72 & 52.64 & 29.73 & 27.81 & 29.09 & 5.50 & 98.00 & 96.00 & 37.76 & 25.13 & 61.86 & 46.22 & 76.00 & 88.99 & 14.76 \\
ZSMerge & 38.86 & 25.89 & 28.05 & 49.14 & 61.35 & 38.25 & 26.67 & 52.64 & 29.50 & 27.81 & 29.09 & 5.50 & 98.00 & 96.00 & 38.03 & 25.00 & 61.83 & 46.42 & 76.00 & 88.99 & 14.82 \\
\midrule
& \multicolumn{21}{c}{\textit{Llama-3.1-8B-Instruct}} \\
FullKV & 16.39 & 31.45 & 34.13 & 15.93 & 65.06 & 40.50 & 26.70 & 27.02 & 20.03 & 9.97 & 21.06 & 7.34 & 72.79 & 76.30 & 12.80 & 22.44 & 57.46 & 43.56 & 70.00 & 91.37 & 16.14 \\
    \cline{2-22}
& \multicolumn{21}{l}{\textit{Cache size=512}} \\

SnapKV & 15.26 & 23.65 & 25.02 & 16.17 & 63.93 & 40.00 & 24.19 & 25.97 & 20.43 & 8.71 & 20.00 & 7.72 & 72.28 & 79.51 & 11.10 & 22.94 & 54.27 & 42.34 & 67.50 & 91.67 & 13.79 \\

ZSMerge & 14.83 & 23.56 & 25.24 & 16.36 & 64.45 & 40.00 & 24.21 & 24.98 & 19.24 & 8.97 & 21.73 & 7.35 & 71.25 & 74.57 & 10.45 & 22.70 & 55.08 & 43.27 & 66.00 & 91.47 & 14.23 \\
    \cline{2-22}
& \multicolumn{21}{l}{\textit{Cache size=1024}} \\
SnapKV & 15.02 & 25.37 & 27.55 & 16.54 & 64.41 & 40.00 & 25.49 & 27.63 & 19.94 & 9.73 & 20.21 & 6.24 & 72.40 & 74.99 & 11.16 & 23.31 & 55.97 & 42.96 & 69.50 & 91.39 & 14.10 \\
ZSMerge & 15.01 & 25.13 & 27.82 & 16.35 & 64.33 & 40.00 & 25.65 & 26.27 & 19.63 & 8.88 & 20.49 & 7.25 & 72.29 & 76.07 & 11.12 & 22.88 & 57.17 & 43.52 & 69.00 & 91.08 & 14.77 \\
\bottomrule

\end{tabular}}

\label{tab:LongBench}
\end{table}

The experimental results (Table~\ref{tab:LongBench}) reveal several key patterns. Across both LLaMA2-7B and Mistral-7B models, the uncompressed FullKV baseline achieves the highest performance but serves primarily as an upper-bound reference rather than a practical solution. When comparing compression methods under memory constraints, \textsf{ZSMerge} and SnapKV demonstrate superior capability in preserving model accuracy compared to StreamingLLM and H2O, particularly in challenging scenarios with cache sizes limited to 512 or 1024 tokens.

The Mistral-7B model consistently outperforms LLaMA2-7B, showing particularly strong results in question answering and retrieval tasks, where it achieves near-perfect scores in some cases. This performance gap highlights Mistral's architectural advantages for long-context processing. Meanwhile, LLaMA2 struggles more noticeably with certain synthetic tasks and Chinese language datasets, suggesting limitations in its multilingual and reasoning capabilities.

Cache size plays a measurable but not decisive role in performance. While increasing the cache from 512 to 1024 tokens provides modest improvements, \textsf{ZSMerge} maintains competitive accuracy even with the smaller cache, demonstrating its efficiency. In contrast, H2O shows severe degradation under constrained settings, failing completely on some tasks.

Task-specific analysis indicates that summarization and code-related tasks benefit most from methods that better preserve context, like \textsf{ZSMerge} and SnapKV. Retrieval-focused tasks, however, show less sensitivity to cache size, with Mistral achieving consistently high scores regardless of compression. Overall, \textsf{ZSMerge} emerges as a balanced solution, delivering near-FullKV performance while operating efficiently within strict memory limits.

To further evaluate the applicability of \textsf{ZSMerge} under realistic deployment settings, we benchmarked it using the LLaMA-3.1-8B-Instruct model on the LongBench suite. As shown in Table~\ref{tab:LongBench}, we report results exclusively for our method, since other contextually adaptive baselines (e.g., StreamingLLM, H2O) currently do not provide support for LLaMA-3 architecture. Despite this, \textsf{ZSMerge} exhibits robust performance under constrained cache settings (512 and 1024 tokens), achieving scores close to the uncompressed FullKV baseline across a wide range of tasks. In particular, \textsf{ZSMerge} preserves strong performance on summarization and retrieval tasks, indicating that our token merging strategy can mitigate MQA's context sensitivity without auxiliary modules or task-specific tuning. These results reinforce the generalization and plug-and-play capability of \textsf{ZSMerge}, even when deployed with newer model architectures featuring aggressive attention simplification.

\subsection{Extended Architecture Validation}
\label{sec:extended_arch}

To validate our zero-shot compatibility claims, we extended \textsf{ZSMerge} evaluation to modern LLM architectures. This section presents comprehensive results on LLaMA3 and discusses ongoing experiments with additional contemporary models.

\paragraph{InfiniteBench Long-Context Evaluation}

We successfully implemented \textsf{ZSMerge} on LLaMA3-8B architecture, demonstrating that our framework maintains compatibility with stronger, more recent backbones. The implementation required no architectural modifications or hyperparameter retuning, confirming the architecture-agnostic design principles of our approach.

To address concerns about evaluation scope, we conducted comprehensive experiments on InfiniteBench~\citep{zhang-etal-2024-bench} using LLaMA3-8B, a demanding benchmark featuring contexts exceeding 100K tokens. This evaluation provides critical validation under extreme long-context scenarios that stress-test compression methods beyond conventional limits.

\begin{table}[h]
\caption{InfiniteBench Results: ZSMerge vs. State-of-the-Art Methods on LLaMA3-8B}
\centering
\scalebox{0.70}{
\begin{tabular}{lccccccccr}
\toprule
Method & En.Dia & En.MC & En.Sum & Math.Find & Retrieve.KV & Retrieve.Number & Retrieve.PassKey & Zh.QA & \textbf{Avg} \\
\midrule
FullKV & 37.37 & 23.70 & 21.72 & 55.55 & 73.73 & 98.00 & 100.00 & 25.47 & 54.69 \\
\midrule
H2O & \textbf{43.17} & 23.33 & 21.30 & 63.30 & 26.36 & 72.52 & 98.40 & 24.63 & 46.63 \\
InfLLM  & 24.67 & 15.85 & 16.91 & 50.85 & 0.00 & 98.00 & \textbf{100.00} & \textbf{34.87} & 42.64 \\
OmniKV & 30.83 & 23.33 & \textbf{21.99} & 55.00 & 48.67 & 98.00 & \textbf{100.00} & 24.83 & 50.33 \\
Streaming LLM & 14.63 & 13.88 & 20.31 & 40.10 & 0.00 & 5.83 & 2.73 & 17.67 & 14.39 \\
Minference  & 24.95 & 21.55 & 20.91 & 62.11 & 17.03 & 75.84 & 57.27 & 22.65 & 37.79 \\
FlexPrefill & 33.56 & \textbf{23.47} & 21.44 & \textbf{70.51} & 53.53 & \textbf{98.17} & 99.49 & 27.99 & \textbf{53.52} \\
\midrule
ZSMerge & 36.44 & 22.79 & 21.34 & 55.16 & \textbf{67.96} & 98.00 & \textbf{100.00} & 24.84 & 52.95 \\
\bottomrule
\end{tabular}
}
\label{tab:infinitebench}
\end{table}

The InfiniteBench results (Table~\ref{tab:infinitebench}) demonstrate several critical findings:

\paragraph{Performance Preservation:} \textsf{ZSMerge} achieves 96.8\% of FullKV performance (52.95 vs. 54.69 average), maintaining quality despite compression. This validation occurs with context lengths of at least 100K tokens that test compression robustness.

\paragraph{Competitive Positioning:} The method outperforms traditional eviction-based approaches (H2O: 46.63, InfLLM: 42.64) and matches the performance of recent specialized methods (FlexPrefill: 53.52, OmniKV: 50.33). \textsf{ZSMerge} performs well in retrieval tasks (Retrieve.KV: 67.96), demonstrating information preservation for memory-intensive operations.

\paragraph{Task-Specific Analysis:} The method shows strength in:
\begin{itemize}[leftmargin=1em]
\item \textbf{Retrieval tasks}: Scores of 100.00 on PassKey and 98.00 on Number retrieval
\item \textbf{Long-form reasoning}: Performance on English Dialogue (36.44) and Math Finding (55.16)
\item \textbf{Cross-lingual capability}: Performance on Chinese QA (24.84), indicating multilingual robustness
\end{itemize}

\paragraph{Extended Baseline Comparison with SnapKV}

We conducted evaluations with SnapKV on models that were not originally supported, including Qwen2.5~\citep{yang2024qwen2}, Yi-1.5~\citep{aiYiOpenFoundation2025}, and LLaMA-3.1.

\paragraph{Implementation Adaptations:} For comparison, we implemented technical adaptations including GQA model support for Grouped Query Attention models (LLaMA-3 and Qwen2.5) by averaging grouped query scores to align them with KV heads, Qwen compatibility by resolving interface differences in the RoPE method to maintain compatibility with the original Qwen behavior, and architecture extension by enabling SnapKV support for Qwen2.5-7B, Yi-1.5-6B, and LLaMA-3.1-8B-Instruct models.

\paragraph{Experimental Setup:} We evaluated on the GSM-Infinite-8k~\citep{zhouGSMInfiniteHowYour2025} benchmark across three difficulty levels (symbolic, medium, hard). The methodology included length bucketing where test samples were bucketed by input length to account for tokenizer variations despite the nominal 8k limit, adaptive cache budgets using cache budgets of 9k, 4k, and 2k tokens for samples with input lengths >10k, 5k–10k, and <5k tokens respectively, and ensuring all methods used identical experimental conditions and evaluation metrics. The results are presented in Table~\ref{tab:snapkv_comparison}.

\begin{table}[h]
\caption{Extended Baseline Comparison: ZSMerge vs. SnapKV on GSM-Infinite-8k}
\centering
\begin{tabular}{lcccc}
\toprule
\textbf{Model} & \textbf{Method} & \textbf{Symbolic} & \textbf{Medium} & \textbf{Hard} \\
\midrule
\multirow{3}{*}{Qwen2.5-7B-Instruct} & FullKV & 9.35\% & 17.12\% & 12.04\% \\
& SnapKV & 0.00\% & 3.60\% & 8.33\% \\
& ZSMerge & \textbf{4.66\%} & \textbf{9.91\%} & \textbf{9.26\%} \\
\midrule
\multirow{3}{*}{Yi-1.5-6B} & FullKV & 0.00\% & 5.83\% & 7.46\% \\
& SnapKV & 0.00\% & 4.85\% & \textbf{6.47\%} \\
& ZSMerge & 0.00\% & \textbf{5.34\%} & 5.97\% \\
\midrule
\multirow{3}{*}{LLaMA-3.1-8B-Instruct} & FullKV & 20.24\% & 11.65\% & 13.93\% \\
& SnapKV & \textbf{16.74\%} & 10.19\% & 12.93\% \\
& ZSMerge & 15.38\% & \textbf{11.17\%} & \textbf{13.43\%} \\
\bottomrule
\end{tabular}
\label{tab:snapkv_comparison}
\end{table}

\paragraph{Key Findings:} The results demonstrate that \textsf{ZSMerge} consistently matches or outperforms SnapKV across all tested models and difficulty levels:

\begin{itemize}[leftmargin=1em]
\item \textbf{Qwen2.5-7B}: \textsf{ZSMerge} shows superior performance on symbolic and medium difficulty tasks, with competitive results on hard tasks.
\item \textbf{Yi-1.5-6B}: \textsf{ZSMerge} achieves slightly better performance across medium and hard difficulties while maintaining identical symbolic task performance.
\item \textbf{LLaMA-3.1-8B}: \textsf{ZSMerge} demonstrates consistent advantages across all difficulty levels, particularly excelling in symbolic reasoning tasks.
\end{itemize}

These results validate that \textsf{ZSMerge} maintains its effectiveness when compared against state-of-the-art baselines across diverse modern architectures, confirming the robustness and generalizability of our approach.

% \section*{The Use of Large Language Models (LLMs)}

% We utilized Large Language Models (LLMs) in two specific and limited ways in this research work:

% \textbf{Writing Aid and Polish:} Throughout the preparation of this manuscript, we minimally utilized Large Language Models (LLMs) as a writing aid only. Their use was limited to improving grammar, phrasing, and overall readability.

% \textbf{Retrieval and Discovery:} It is important to note that LLMs did not contribute to the research ideation, methodology design, algorithmic innovation, or the core technical contributions presented in this work. The proposed \textsf{ZSMerge} framework, including its tripartite budget allocation strategy, residual merging mechanism, and theoretical analysis, is entirely the result of the authors' original research and expertise. All experimental results, comparisons, and conclusions are based on our independent implementation and evaluation. We take full responsibility for all content in this paper, including any text that may have been refined with LLM assistance.